\documentclass{article} % For LaTeX2e
\usepackage{iclr2023_conference_tinypaper,times}

\usepackage{amsmath,amsfonts,bm}

\def\eqref#1{equation~\ref{#1}}
\def\1{\bm{1}}

\def\eps{{\epsilon}}

\DeclareMathAlphabet{\mathsfit}{\encodingdefault}{\sfdefault}{m}{sl}
\SetMathAlphabet{\mathsfit}{bold}{\encodingdefault}{\sfdefault}{bx}{n}

\def\tY{{\tens{Y}}}

\DeclareMathOperator*{\argmax}{arg\,max}
\DeclareMathOperator*{\argmin}{arg\,min}

\usepackage{hyperref}
\usepackage{url}
\usepackage{amssymb,nicefrac,amsthm,bbm}

\def\argmin{{\arg\min}}
\def\argmax{{\arg\max}}

\def\bI{\mathbf{I}}

\def\bbE{\mathbb{E}}

\def\bbR{\mathbb{R}}

\def\cA{\mathcal{A}}

\def\cF{\mathcal{F}}

\def\cQ{\mathcal{Q}}

\def\cX{\mathcal{X}}
\def\cY{\mathcal{Y}}

\def\Ex{{\bbE}}

\def\eps{{\epsilon}}

\def\ie{{\it i.e.}}
\def\iid{{\sf iid}}

\def\reals{{\bbR}}

\newtheorem{theorem}{Theorem}[section]

\newtheorem{lemma}[theorem]{Lemma}

\def\tY{{Y'}}
\usepackage{cancel}
\usepackage{graphicx}

\RequirePackage{xcolor}

\newenvironment{blue}{\color{blue}}{}

\title{\vspace{-1cm}
Bayes classifier cannot be learned from noisy responses with unknown noise rates}

\author{Soham Bakshi\thanks{Equal contributions.}\\
Department of Statistics\\
University of Michigan\\
\texttt{baksho@umich.edu} \\
\And 
Subha Maity\textsuperscript{$*$}\\
Department of Statistics\\
University of Michigan\\
\texttt{smaity@umich.edu} \\
}

\iclrfinalcopy % Uncomment for camera-ready version, but NOT for submission.
\begin{document}

\maketitle

\vspace{-0.7cm}

\begin{abstract}
Training a classifier with noisy labels typically requires the learner to specify the distribution of label noise, which is often unknown in practice. Although there have been some recent attempts to  relax that requirement, we show that the Bayes decision rule is unidentified in most classification problems with noisy labels. This suggests it is generally not possible to bypass/relax the requirement. In the special cases in which the Bayes decision rule is identified, we develop a simple algorithm to learn the Bayes decision rule, that does not require knowledge of the noise distribution. 
% remains a challenging problem to date. While most existing approaches require practitioners to specify the {noise rates} (a set of parameters that control the severity of noise in the response variable), some attempts have recently been made to relax that requirement. In this note we investigate a more fundamental question: \emph{whether it is at all possible to learn the Bayes classifier} for the clean response, \emph{when the noise rates are unknown}. We provide a comprehensive answer by studying the \emph{identifiability} of the Bayes classifier and find that in most cases {the task is theoretically impossible} apart from the special case of binary classification with balanced classes in which scenario a classifier learned from the noisy dataset using a weighted empirical risk minimization (ERM) {with any consistent loss function} produces a consistent estimator of the true Bayes classifier. 
\end{abstract}

\vspace{-0.3cm}
\section{Introduction}
% \vspace{-0.1cm}

In this paper, we consider classification with noisy labels. Let $\cX$ and $\cY$ be the feature/input and label/output spaces, respectively. The clean/noiseless samples $(X_i,Y_i)$ are drawn independently from $P_{X,Y}\in\Delta(\cX\times\cY)$ ($\Delta(\cA)$ is the space of probability measures on $\cA$), but the learner only observes the $(X_i, Y_i')$'s,  where $Y_i'$ is corrupted version of $Y_i$ from a conditional distribution: $Y'\mid Y\sim P_{Y'\mid Y}$. The learner seeks to estimate the Bayes classifier of $Y$ (the clean/noiseless label) 
\[\textstyle
f^\star _P(x) \triangleq \argmax_{y\in\cY} P(Y = y\mid X=x)
\]
from the noisy training data $\{(X_i,Y_i')\}_{i=1}^n$. The classification with noisy labels problem arises in many areas of science and engineering, including medical image analysis \citep{karimi2020deep} and crowdsourcing \citep{jiang2021learning}. 

When the label noise rates/distribution $P_{Y'\mid Y}$ is known or learnable from external data, there are several methods to recover $f^\star _P$ \citep{bylander1994learning,cesa1999sample,natarajan2013learning}. Unfortunately, $P_{Y'\mid Y}$ is often unknown to the learner in practice, which limits their applicability. 
% \YK{Discuss relevant papers here in here (instead of in separate ``Related work'' subsection)}
Recently, \citet{liu2020peer} propose a method based on \emph{peer prediction}, which provably recovers $f^\star _P$ when there are only two classes and they are balanced (but the label noise distribution is unknown). 
The framework of peer prediction is motivated from  \citet{dasgupta2013crowdsourced,shnayder2016informed}, and has been further developed for ranking problems with noisy labels \citep{wu2022peerrank}. A review on learning from noisy labels and peer prediction can be found in Appendix \ref{sec:related-work}.

% \YK{This guy also has some other papers that develop/extend the peer prediction idea further. Please cite the relevant appers here.} 

In this paper, we consider the statistical aspects of classification with noisy labels. Our main contributions are the following.\vspace{-0.1cm}
\begin{itemize}
   \item We show that the balanced binary classification problem is the only instance in which $f^\star _P$ can be learned without knowledge of the label noise distribution, while in more general problems (with imbalanced or more than two classes), that knowledge is essential.
\item We develop a new method based on weighted empirical risk minimization (ERM) that provably learns $f^\star _P$ in the balanced binary classification problem with noisy labels.
\end{itemize}

\vspace{-0.2cm}

\section{Identifiability of the Bayes classifier}
\vspace{-0.2cm}
In our setup a typical data-point $(X, Y, Y')$ (a triplet of feature, clean label and noisy label) comes from a true distribution $P \equiv P_{X, Y, \tY}$, whose full joint distribution is unknown. Since the learner only observes \iid\  $(X_i, Y'_i)$ pairs we assume that the  $P_{X, \tY}$ marginal is known.   Furthermore, we assume that the noise rates/distributions are \emph{instance independent}, \ie,  for any $x \in \cX$ and $y, y' \in \cY $ 
\begin{equation} \label{eq:ii}
 \textstyle   P(Y' = y' \mid Y = y, X = x) = P(Y' = y' \mid Y = y) \triangleq \eps_P(y', y)\,.
\end{equation} 
\vspace{-1cm}

% The Bayes classifier for clean label of the distribution $P$ is denoted as $f^\star_P$ and defined as the classifier with minimum misclassification rate, \ie, $f^\star_P \triangleq \argmin_{f : \cX \to \cY} P(f(X) \neq Y)$. 
% \begin{equation}
%     f^\star_P \triangleq \argmin_{f : \cX \to \cY} P(f(X) \neq Y)\,.
% \end{equation} 

% For our investigation we assume that the marginal distribution of $Y$ is $p \in \Delta^{\cY}$ (\ie, $p$ is a probability vector on $\cY$), and perform separate investigations for each $K \ge 2$ and $p \in \Delta^{\cY}$. 

In our investigation we fix the marginal of $Y$, \ie, for some $p \in \Delta(\cY)$ we assume that  $P_Y = p$. 
Thus, we define the class $\cQ(K, p)$ of all probabilities $Q \equiv Q_{X, Y, Y'}\in \Delta (\cX \times \cY \times \cY)$ that satisfy (1) $Q_{X, Y'} = P_{X, Y'}$, (2) has instance independent noise (\ie, $Q$ satisfies \eqref{eq:ii}), (3) $Q_Y = p$, and (4) the determinant of $E_Q = [[\eps_{Q}(y', y)]]_{y', y\in \cY}$ is positive. The final condition is a regularity condition on the noise rates and is satisfied if $Q(Y' \neq Y \mid Y =  y)$ are not too large. In fact, for binary classification it boils down to $\eps_{Q}(0, 1) + \eps_Q(1, 0) <1$, which is a rather weak assumption and standard in the literature \citep{natarajan2013learning,liu2020peer}.

In the following theorem, whose proof can be found in Appendix \ref{sec:non-identifiability-proof}, we investigate whether the Bayes classifier $f_Q^\star(x) \triangleq \argmin_{y \in \cY} Q(Y = y\mid X = x)$ is same for all the $Q\in \cQ(K, p)$.

\begin{theorem}[Identifiability of the Bayes classifier]
\label{th:identifiability}
 The Bayes classifier $f_Q^\star$ is unique for all $Q\in \cQ(K, p)$, \ie,\  $\{f_Q^\star: Q\in \cQ(K, p)\}$ is a singleton set \emph{if and only if} $K = 2$ and $p = (\nicefrac12, \nicefrac12)$. 
\end{theorem} 

\vspace{-0.1cm}

\subsection{The identifiable case}

% \vspace{-0.1cm}

Balanced binary classification \ie, $p = (\nicefrac12, \nicefrac12)$ is the special case when \emph{the Bayes classifier is unique, regardless of the noise rates}. This is the optimistic case where the noise rates are not required for learning $f^\star_{P}$. In this case \emph{we provide an alternative} to the peer loss framework \citep{liu2020peer} for learning $f_P^\star$ that relies on the popular weighted ERM method:
\begin{equation} \label{eq:weighted-erm}
  \textstyle  \min_{\eta \in \mathcal{F}}\frac{1}{n}  \sum_{i = 1}^n p_n(1 - Y'_i) \ell(\eta(X_i), Y_i') \,,
\end{equation} 
where $\mathcal F$ is a set of probabilistic classification models such that for some $\eta^\star \in \mathcal{F}$ it holds  $f_P^{\star}(x) = \mathbbm{1}\{\eta^\star(x) \ge \nicefrac{1}{2}\}$, $\ell$ is an appropriate loss function, and $ p_n(y) \triangleq \nicefrac{1}{n} \sum_{i = 1}^n \mathbbm{1}\{Y'_i = y\}$. The following lemma, whose proof can be found in Appendix \ref{sec:technical-results}, establishes that the weighted ERM in \eqref{eq:weighted-erm} using the noisy distribution ($P_{X, Y'}$) estimates $f_P^\star$, regardless of the noise rates.
\begin{lemma}[Weighted ERM]
\label{lemma:reweighting-informal}
Let $Y, Y' \in \{0, 1\}$, $P(Y = 1) = \nicefrac{1}{2}$, $\cF$ be the set of all binary classifiers on $\cX$,  $\ell(f(x) , y ) = \mathbbm{1}\{f(x) \neq y\}$ for $f\in \cF$  and $\eps_{P}(0, 1) + \eps_P(1, 0) <1$.
If $p'(1) = P(Y' = 1)$ and $p'(0) = P(Y' = 0)$ then regardless the values of $\eps_{P}(0, 1)$ and $\eps_P(1, 0)$ the Bayes classifier is recovered, \ie  \begin{equation}
    f^\star_P(x)= \argmin_{f \in \mathcal{F}} \Ex_P[p'(1 - Y') \ell(f(X), Y')]\,.
\end{equation}

% If $Q$ is constructed by reweighting $P$ with respect to the weights $w(x, y, y') \propto P(Y' = 1-y')$ then regardless of $\eps_{P}(0, 1)$ and $\eps_P(1, 0)$ the Bayes classifier for predicting $Y'$ in $Q$ is same as $f_P^\star$ \ie, 
% \begin{equation} \textstyle
% % \argmax_{y \in \{0, 1\}} Q(Y' = y\mid X = x) = \argmax_{y \in \{0, 1\}} P(Y = y\mid X = x)\,.
%   \{x: Q(Y' = 1\mid X = x) \ge  \nicefrac{1}{2} \}  = \{x: P(Y = 1\mid X = x) \ge  \nicefrac{1}{2} \}\,.
% \end{equation}
\end{lemma}
% an  weighted ERM with any consistent loss function produces a classifier with desirable statistical consistency.
Thus, the weighted ERM is an alternative framework to the peer loss \citep{liu2020peer} for learning $f_P^\star$ from noisy labels. In Appendix \ref{sec:peer-loss} we compare these two frameworks, where we also highlight a drawback of the peer loss, that it may not be bounded below and may diverge to $-\infty$ while minimization. Our method, on the other hand, does not suffer from it. 

% \SM{add the weighted ERM equation here. do we need the criticism of the peer loss in main draft?}

Though the noise rates are not needed for the aforementioned case,  it is impossible to verify whether $P(Y = 1) = \nicefrac{1}{2}$ if nothing is known about $P_Y$. So, the weighted ERM may not be practical without precise information about $P_Y$, which is also required by the peer loss framework. 

% \vspace{-0.1cm}

\subsection{The non-identifiable cases}

% \vspace{-0.1cm}

For imbalanced binary classification or with more than two classes the Bayes classifier is not identifiable when the noise rates ($P_{Y' \mid Y}$) are unknown. In fact, for establishing the proof of Theorem \ref{th:identifiability} we construct two different $P_{Y'\mid Y}$'s that are compatible with the marginals $P_{X, Y'}$ and $P_Y$ but have different Bayes decision boundaries. This is the problematic case, where it is statistically impossible to learn the Bayes classifier owing to lack of identifiability, and \emph{an additional knowledge on $P_{Y'\mid Y}$} is essential for developing meaningful procedures.  

\section{Discussion}

% \vspace{-0.1cm}

In this study, we present a thorough examination of the identifiability of the Bayes classifier in classification scenarios with noisy labels where the noise rates are  unknown.
The necessity of knowing the noise rates (in general) is clear from our results: in almost all cases, it is impossible to learn the Bayes classifier for the true labels without this piece of information. We hope that our findings can help practitioners develop a better understanding about the limitations and requirements for learning classification models from noisy labels.

\subsubsection*{URM Statement}
The authors acknowledge that both the authors of this work meet the URM criteria of ICLR 2023 Tiny Papers Track.

\subsubsection*{Acknowledgments}
The authors would like to thank Prof. Yuekai Sun and Prof. Moulinath Banerjee for their insightful comments and discussions related to this work. 
Subha Maity was supported by the National Science Foundation (NSF) under grants no.\ 2027737 and 2113373 while working on this paper. 

\bibliography{sm}
\bibliographystyle{iclr2023_conference_tinypaper}
% \newpage
\appendix
\section{Related work}
\label{sec:related-work}
 \textbf{Learning with noisy response} has been a topic of great importance many areas of science and engineering, including medical image analysis \citep{karimi2020deep} and crowdsourcing \citep{jiang2021learning}.  It has produced a wide range of researches related to importance re-weighting algorithm \citep{liu2015classification}, robust cross-entropy loss for neural networks \citep{zhang2018generalized},  loss correction \citep{patrini2017making}, learning noise rates \citep{liu2015classification,patrini2017making,xiao2015learning}. A more comprehensive reviews about the literature can be found in \citet{song2022learning,liang2022review}.

Our work considers the classification problems with noisy labels when the noise rates are \textbf{instance independent}, as specified in \eqref{eq:ii}. When the noise rates are known, various methods for learning a classifier have been proposed by \citet{bylander1994learning,cesa1999sample,natarajan2013learning}. \citet{liu2020peer} recently introduced the peer loss approach for learning a classifier for binary classification problem without prior knowledge of the noise rates. However, it remains unclear whether this approach can be extended beyond binary classification. Our work addresses this gap and complements the findings by \citet{liu2020peer} in two ways: (1) our results explain why a classification task can be performed for balanced binary classification problems without requiring knowledge of the noise rates, and (2) we demonstrate that this is the only scenario in which the Bayes classifier is uniquely identified and a statistically consistent \cite{bartlett2006convexity} classification is possible.

\textbf{Peer loss} has been proposed by \citet{liu2020peer} for learning the classifier without knowing the noise rates, which is most related to this work. The peer loss has been motivated from the ideas of peer prediction \citep{prelec2004bayesian,miller2005eliciting,witkowski2012robust,dasgupta2013crowdsourced,witkowski2013dwelling,radanovic2016incentives,shnayder2016informed}, and has been further developed for ranking problems \citep{wu2022peerrank}.

% Our work complements the findings by \citet{liu2020peer} on two fronts: (1) our results reveal why it is possible to perform a classification task for balanced binary classification problems with noisy labels without requiring noise rates, a quantity that is required in previous works by \citet{natarajan2013learning}

% In this work we investigate the underspecification of the \textbf{Bayes classifier} for the clean labels. The Bayes classifier, which has minimum misclassification error among all the classifiers, is fundamental quantity of interest for any classification problem and all the classification methods e.g., decision tree, random forest, boosting, SVM consistently approximates/estimates the Bayes classifier \cite{bartlett2006convexity}. 

% \section{A list of notations}
\section{Synthetic experiment}

We empirically investigate the classification approach in \eqref{eq:weighted-erm} on a synthetic dataset for binary classification, whose description follows: (1) $Y \sim \text{Bernoulli}(p)$, (2) $X \mid (Y = y) \sim \textsc{N}_2(\nicefrac{1}{2} (2y-1) \mathbf{1}_2, \bI_2)$, and (3) $P(Y' = 0 \mid Y = 1) = \eps_1 $, $P(Y' = 1 \mid Y = 0) \eps_0$. We consider two situations: (1) the balanced case ($p = \nicefrac{1}{2}$), when the Bayes classifier is identified, and (2) an imbalanced case with $p = 0.35$, when the Bayes classifier is not identified. For both the cases we set $\eps_1 = \nicefrac{\{0.4 - (1 - \eps_0) (1 - p)\}}{p}$. Note that, for such a choice 
\[
P(Y ' = 0) = P(Y' = 0\mid Y = 1) p + P(Y' = 0\mid Y = 0) (1 - p) = \eps_1 p + (1 -\eps_0) (1 - p) = 0.4\,. 
\] We compare the classification approach in \eqref{eq:weighted-erm} with two baselines: (1) the \emph{oracle}, trained with the true $Y$, and (2) the \emph{baseline}, trained the noisy $Y'$ without any adaptation. We use logistic regression model for all the classifiers. We do not consider peer-loss \cite{liu2020peer} as a baseline for its failure case described in Appendix \ref{sec:peer-loss}.

In the balanced case (left plot in Figure \ref{fig:expt}) we see that the class balancing approach in \eqref{eq:weighted-erm} has identical performance to the oracle/ideal case for large sample sizes ($n_{\text{tr}} = 2500$), which is an evidence for it's statistical consistency \cite{lugosi2004bayes}. On contrary, for imbalanced case (right plot in the same figure) there is a gap between the  class balancing approach and the oracle, even for large sample sizes, \ie, the  class balancing approach is statistically inconsistent. 

\begin{figure}
    \centering
    \includegraphics[width = 0.7\textwidth]{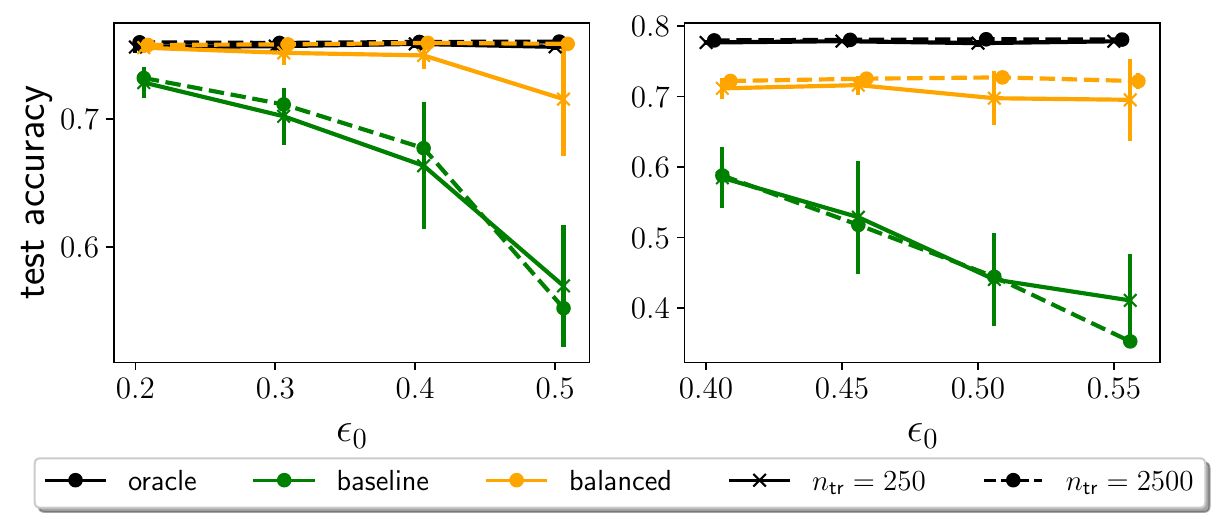}
    \caption{The consistency (resp. inconsistency) of class balancing approach in \eqref{eq:weighted-erm} for balanced (resp. imbalanced) binary classification, as observed in left (resp. right) plot.}
    \label{fig:expt}
\end{figure}
\section{Proof of Theorem \ref{th:identifiability}}
\label{sec:non-identifiability-proof}

For readers convenience we restate the Theorem \ref{th:identifiability}.
\begin{theorem}[Identifiability of the Bayes classifier]

The Bayes classifier $f_Q^\star$ is unique for all $Q\in \cQ(K, p)$, \ie,\  $\{f_Q^\star: Q\in \cQ(K, p)\}$ is a singleton set \emph{if and only if} $K = 2$ and $p = (\nicefrac12, \nicefrac12)$. 
\end{theorem} 

For $Q\in \cQ(K, p)$ we recall the definition of the Bayes classifier that 
\[
f_Q^\star (x) = \argmax_{k \in \cY} Q(Y = k\mid X = x) = \argmax_{k \in \cY} Q(Y = k\mid X = x) p_X(x)
\] where $p_X$ is the density function of $X$. Note that the marginal of $X$ are same for $P$ and $Q$.  Henceforth, we define the class probabilities of $Y'$ as $Q_{Y'} = P_{Y'} = p'$,
\[
\begin{aligned}
\alpha_k(x) & = Q_{X, Y}(x, k) \triangleq Q(Y= k \mid X = x) p_X(x),\\
a_{k}(x) &= Q_{X, Y'}(x, k) \triangleq Q(Y'= k \mid X = x) p_X(x)\,.
\end{aligned}
\]
 We further recall that $\eps(k',k) = Q_{Y'|Y}(Y' = k'\mid Y = k)$. 
 
% First of all note that identifying the Bayes classifier $f^{*}_{Q}$ is equivalent to identifying Bayes decision boundary which basically depends on the ordering of $\{Q_{(X,Y)}(x,k)\}_{k \in [K]}$ for each $x \in \mathcal{X}$. To be more precise, the Bayes classifier $f^{*}_{Q}(x) = k$ if $Q_{(X,Y)}(x,k) > Q_{(X,Y)}(x,k')$ for all $k' \not= k \in [K]$. Hence, $\{f_Q^\star: Q\in \cQ(K, p)\}$ is a singleton set if and only if for each $x \in \cX$, $\text{argmax}_{k \in [K]}Q_{(X,Y)}(x,k)$ is same for all $Q \in \cQ(K,p)$. So, basically from now on to show that the Bayes classifier is not unique, we will show that we can change $\argmax_{k \in [K]}Q_{(X,Y)}(x,k)$, changing the error rates of $Q$ but still satisfying all the properties of $\cQ$.

% \textbf{Notations:} Say, $Q_{Y'} = P^{*}_{Y'} = p'$ (the class probabilities of $Y'$) and $\eps(k',k) = Q_{Y'|Y}(k'|k)$. We also write $\alpha_{k}(x) = Q_{(X,Y)}(x,k)$ and $a_{k}(x) = Q_{(X,Y')}(x,k)$ for all $k \in [K]$ and $x \in \cX$. 

\subsection{The binary classification case} 
Here we let $\cY = \{1, 2\}$ and establish that the Bayes decision boundary is unique if and only if $p = (\nicefrac12, \nicefrac12)$. We begin with a lemma that we require for the proof.

% \textbf{Case 1:} $\boxed{K = 2}$
% This is the case of binary classification and we shall prove that the Bayes decision boundary is unique if and only if $p = (\nicefrac12, \nicefrac12)$.

\begin{lemma} \label{lemma:invert}
The following holds 
\begin{equation}\label{eq:invert}
    \begin{bmatrix} \alpha_{1}(x)\\ \alpha_{2}(x) \end{bmatrix} = \frac{1}{1-\eps(1,2) - \eps(2,1)}\begin{bmatrix} 1-\eps(1,2) & -\eps(1,2) \\ -\eps(2,1) & 1-\eps(2,1) \end{bmatrix} \begin{bmatrix} a_{1}(x)\\ a_{2}(x) \end{bmatrix}\,.
\end{equation}

% $\begin{bmatrix}
% \alpha_{1}(x)\\
% \alpha_{2}(x)
% \end{bmatrix}$ = $E \begin{bmatrix} a_{1}(x)\\ a_{2}(x) \end{bmatrix}$, where 
% $E = \frac{1}{1-\eps(1,2)-\eps(2,1)}
% \begin{bmatrix}
% 1-\eps(1,2) & -\eps(1,2)\\
% -\eps(2,1) & 1-\eps(2,1)
% \end{bmatrix}$. 
\end{lemma}

\begin{proof}[Proof of Lemma \ref{lemma:invert}]
  Let's start with $a_{1}(x)$, by definition $a_{1}(x) = Q_{(X,Y')}(x,1)$.
 \[
 \begin{aligned}
     a_{1}(x) & = Q_{(X,Y')}(x,1) = Q_{(X,Y,Y')}(x,1,1)+ Q_{(X,Y,Y')}(x,2,1) \\
     & = Q_{X}(x)Q_{Y|X}(1|x)Q_{Y'|(X,Y)}(1|(x,1)) + Q_{X}(x)Q_{Y|X}(2|x)Q_{Y'|(X,Y)}(1|(x,2)) \\
     & = Q_{X}(x)Q_{Y|X}(1|x)Q_{Y'|Y}(1|1) + Q_{X}(x)Q_{Y|X}(2|x)Q_{Y'|Y}(1|2) \\
     & = Q_{(X,Y)}(x,1)Q_{Y'|Y}(1|1) + Q_{(X,Y)}(x,2)Q_{Y'|Y}(1|2) \\
     & = (1-\eps(2,1))\alpha_{1}(x) + \eps(1,2)\alpha_{2}(x),
     \end{aligned}
 \]
where in the third equality we use $X \perp Y' \mid Y$, that the noise rates are instance independent.  Similarly for $a_{2}(x) = Q_{(X,Y')}(x,2)$ we have 
\[
a_{2}(x)= \eps(2,1)\alpha_{1}(x) + (1-\eps(1,2))\alpha_{2}(x)\,.
\]
%  \[
%  \begin{aligned}
%      a_{2}(x) & = Q_{(X,Y')}(x,2) = Q_{(X,Y,Y')}(x,1,2)+ Q_{(X,Y,Y')}(x,2,2) \\
%      & = Q_{X}(x)Q_{Y|X}(1|x)Q_{Y'|(X,Y)}(2|(x,1)) + Q_{X}(x)Q_{Y|X}(2|x)Q_{Y'|(X,Y)}(2|(x,2)) \\
%      & = Q_{X}(x)Q_{Y|X}(1|x)Q_{Y'|Y}(2|1) + Q_{X}(x)Q_{Y|X}(2|x)Q_{Y'|Y}(2|2) \\
%      & = Q_{(X,Y)}(x,1)Q_{Y'|Y}(2|1) + Q_{(X,Y)}(x,2)Q_{Y'|Y}(2|2) \\
%      & = \eps(2,1)\alpha_{1}(x) + (1-\eps(1,2))\alpha_{2}(x).
%      \end{aligned}
%  \]
 In matrix notation, we have
$\begin{bmatrix} a_{1}(x)\\ a_{2}(x) \end{bmatrix}$ = $\begin{bmatrix} 1-\eps(2,1) & \eps(1,2) \\ \eps(2,1) & 1-\eps(1,2) \end{bmatrix}$ $\begin{bmatrix} \alpha_{1}(x)\\ \alpha_{2}(x) \end{bmatrix}$, which we invert to conclude the proof. 
$$\begin{bmatrix} \alpha_{1}(x)\\ \alpha_{2}(x) \end{bmatrix} = \frac{1}{1-\eps(1,2) - \eps(2,1)}\begin{bmatrix} 1-\eps(1,2) & -\eps(1,2) \\ -\eps(2,1) & 1-\eps(2,1) \end{bmatrix} \begin{bmatrix} a_{1}(x)\\ a_{2}(x) \end{bmatrix}.$$

\end{proof}

Now, $Q_Y = p$ whenever  $Q\in \cQ(2, p)$. Denoting $p_1 = Q(Y = 1)$ and $p_1' =  Q(Y' = 1)$ we notice that 
\[
\begin{aligned}
p_1' &=   Q(Y' = 1|Y = 1)Q(Y = 1) + Q(Y' = 1|Y = 2)Q(Y = 2)\\
& = p_1(1-\eps(2,1) + (1-p_1)\eps(1,2),\\
\text{or,} ~~~~ \eps(2,1) &= \frac{p_1 - p_1' + \eps(1,2) - \eps(1,2)p_1}{p_1}\,.
\end{aligned}
\] 
% whose simplification yields
% \[ \textstyle
% \eps(2,1) = \frac{p_1 - p_1' + \eps(1,2) - \eps(1,2)p_1}{p_1}\,.
% \]
%  Simplifying, $$p = \frac{p' - \eps(1,2)}{1- \eps(1,2) - \eps(2,1)} \iff \eps(2,1) = \frac{p - p' + \eps(1,2) - \eps(1,2)p}{p}.$$
We now use the above equation to replace $\eps(2,1)$ in Lemma \ref{lemma:invert}. Notice that, according to  Lemma \ref{lemma:invert}
% $\alpha_{1}(x)- \alpha_{2}(x)$ in terms of $a_{1}(x)$ and $a_{2}(x)$ as:
\[
\alpha_{1}(x) - \alpha_{2}(x) = \frac{(1- \eps(1,2) + \eps(2,1)) a_{1}(x) - (1- \eps(2,1) + \eps(2,1))a_{2}(x)}{1- \eps(1,2) - \eps(2,1)}\,.
\]
We now plug in the expression of $\eps(2,1)$ in terms of $\eps(1,2)$, $p'$ and $p$, and obtain
\[
\begin{aligned}
& 1- \eps(1,2) + \eps(2,1)= \frac{p - p\eps(1,2) + p - p' +\eps(1,2)- \eps(1,2)p}{p} = 2- \frac{p'+\eps(1,2)(2p-1)}{p}\\
& 1- \eps(2,1) + \eps(2,1) = \frac{p+p\eps(1,2)-p+p'-\eps(1,2)+\eps(1,2)p}{p}=\frac{p'+\eps(1,2)(2p-1)}{p}.
\end{aligned}
\]
Thus, 
\begin{equation} \label{eq:decision-boundary-class-balanced}
    \begin{aligned}
    & (1- \eps(1,2) - \eps(2,1))(\alpha_{1}(x) - \alpha_{2}(x)) \\
    &= (1- \eps(1,2) + \eps(2,1)) a_{1}(x) - (1- \eps(2,1) + \eps(2,1))a_{2}(x) \\
    &= 2a_{1}(x) - \frac{p' + \eps(1,2)(2p-1)}{p}(a_{1}(x)+a_{2}(x))
\end{aligned}
\end{equation}

Since $\eps(1,2) + \eps(2,1) < 1$ we have
\[
\begin{aligned}
f_Q^\star(x) &= \mathbbm{1}\Big\{ \frac{Q_{X, Y}(x, 1)}{Q_{X, Y}(x, 1) + Q_{X, Y}(x, 2)} \ge \frac{1}{2}\Big\}\\ & = \mathbbm{1}\Big\{ \frac{\alpha_1(x)}{\alpha_1(x) + \alpha_2(x)} \ge \frac{1}{2}\Big\} \\
& = \mathbbm{1} \{\alpha_{1}(x) - \beta_{1}(x) \ge 0\}  = \mathbbm{1} \Big\{2a_{1}(x) - \frac{p' + \eps(1,2)(2p-1)}{p}(a_{1}(x)+a_{2}(x))\Big\}
\end{aligned}
\] where $a_k(x) = Q_{X, Y'}(x, k)$, $p$ and $p'$ are determined within $\cQ(2, p)$ class. The only quantity that is not determined is $\epsilon(1,2)$. However, we notice that $f_Q^\star$ is independent of the value of $\epsilon(1,2)$ if and only if $p = (\nicefrac{1}{2}, \nicefrac{1}{2})$. For the binary classification case this concludes $\{f_Q^\star: Q\in \cQ(K, p)\}$ is singleton if and only if $p = (\nicefrac{1}{2}, \nicefrac{1}{2})$. 

% Now, the optimal decision boundary depends on the sign of $\alpha_{1}(x) - \beta_{1}(x)$ which is basically the sign of $2a_{1}(x) - \frac{p' + \eps(1,2)(2p-1)}{p}(a_{1}(x)+a_{2}(x))$, which is free of $\epsilon(1,2)$ only when $p = \frac{1}{2}$. So, if $p = \frac{1}{2}$, the optimal decision boundary is unique and can be determined by the sign of $2a_{1}(x) - \frac{p'}{p}(a_{1}(x) + a_{2}(x))$. Otherwise if $p \not= \frac{1}{2}$, then the optimal decision boundary can be changed by changing $\eps(1,2)$.

\subsection{The multiclass classification case}

For $K \ge 3$ we shall prove that the Bayes decision boundary is never unique. This case is further divides in two subcases: (i) balanced $Y$ \ie, $p =  \nicefrac{\mathbf{1}_K}{K} = \nicefrac{1}{K}(1, 1, \dots, 1)^\top$ and (ii) imbalanced $Y$ \ie, $p \neq \nicefrac{\mathbf{1}_K}{K}$. Similar to the binary case
for $\alpha_{k}(x) =  Q_{(X,Y)}(x,k)$ and $a_{k}(x) =  Q_{(X,Y')}(x,k)$  we have  $$\begin{bmatrix}
    a_{1}(x) \\
    a_{2}(x) \\
    \vdots \\
    a_{K}(x)
\end{bmatrix} = E \begin{bmatrix}
    \alpha_{1}(x) \\
    \alpha_{2}(x) \\
    \vdots \\
    \alpha_{K}(x)
\end{bmatrix}, \text{ where } E = \begin{bmatrix}
    \eps(1,1) & \eps(1,2) & \cdots & \eps(1,K) \\
    \eps(2,1) & \eps(2,2) & \cdots & \eps(2,K) \\
    \vdots & \vdots & \cdots & \vdots \\
    \eps(K,1) & \eps(K,2) & \cdots & \eps(K,K)
\end{bmatrix}\,,$$ which implies 
\[
[\alpha_{1}(x) , \alpha_{2}(x), \dots, \alpha_{K}(x)]^\top = E ^{-1} [a_{1}(x) , a_{2}(x), \dots, a_{K}(x)]^\top\,.
\]
% which implies
% $$\begin{bmatrix}
%     \alpha_{1}(x) \\
%     \alpha_{2}(x) \\
%     \vdots \\
%     \alpha_{K}(x)
% \end{bmatrix} = E^{-1} \begin{bmatrix}
%     a_{1}(x) \\
%     a_{2}(x) \\
%     \vdots \\
%     a_{K}(x)
% \end{bmatrix}, \text{ assuming E is invertible}.$$ 
Note that, the vector $[a_{1}(x) , a_{2}(x), \dots, a_{K}(x)]^\top$ is known to us through the distribution of $P_{X, Y'}$. Additionally, we know $p'$, which is the distribution of $Y'$ and for all the distributions $Q \in \cQ(K, p)$  the distribution of $Y$ is $p$. To establish non-identifiability we shall construct two error metrics $E_1$ and $E_2$ that are (1) stochastic (\ie, has non-negative entries with column sum one),  (2) has positive determinant, (3) satisfies $p' = E_1p$ and $p' = E_2 p$ and (4) has different Bayes decision boundaries.

% To show the above results, one can use the exact same way we got the lemma for the binary case, that is basically using conditional independence properties to write $a_{k}(x)$ in terms of $\alpha_{k}(x)$ and the error rates $\eps(.,.)$. Note, $\eps(k,k) = 1- \sum_{j \not= k}\eps(j,k)$, so sum of each column of $E$ is 1. In rest of the proof 
% we want to show that the optimal decision boundary depends on the error rates, what we do is precisely get two different $E$'s which changes the optimal decision boundary but keeps all other conditions intact. One of the constrains we need to keep intact is $p' = Ep$ where $p', p$ are the class probability vectors of $Y'$ and $Y$ respectively. 

\subsubsection{The balanced case}
If $Y$ is class balanced ($p = \nicefrac{\mathbf{1}_K}{K}$) let $E$ be the error matrix accoridng to Lemma \ref{lemma:stochastic-matrix-existance} that is invertible and satisfies $p' = Ep$. We let $E_{1} = E$ and $E_{2} = E P$, where  the matrix $P$ is an even permutation matrix defined as $P = [e_2, e_3, e_1, e_4, \dots, e_K] $ and $\{e_i\}_{i = 1}^K$ as the standard basis of $\reals^K$. 
% $$P = [e_2, e_3, e_1, e_4, \dots, e_K] 
% \begin{bmatrix}
%     0 & 0 & 1 & 0 & \cdots & 0 \\
%     1 & 0 & 0 & 0 & \cdots & 0 \\
%     0 & 1 & 0 & 0 &\cdots & 0 \\
%     0 & 0 & 0 & 1 & \cdots & 0 \\
%     \vdots & \vdots & \vdots & \vdots & \ddots & 0 \\
%     0 & 0 & 0 & 0 & \cdots & 1 \\
% \end{bmatrix}.$$
Then, $E_{1}p = Ep = p'$ and since $Pp = p$ for any permutation matrix $P$ we have $ E_{2}p = EPp = Ep  = p'$. Defining 
$$[
    \alpha_{1}(x) ,
    \alpha_{2}(x) , 
    \dots , 
    \alpha_{K}(x)
]^\top = E_{1}^{-1} [
    a_{1}(x) ,
    a_{2}(x) , 
    \dots , 
    a_{K}(x)
]^\top $$ we notice that 
$$\begin{bmatrix}
    \tilde{\alpha}_{1}(x) \\
    \tilde{\alpha}_{2}(x) \\
    \tilde{\alpha}_{3}(x) \\
    \vdots \\
    \tilde{\alpha}_{K}(x)
\end{bmatrix} = E_{2}^{-1} \begin{bmatrix}
    a_{1}(x) \\
    a_{2}(x) \\
    \vdots \\
    a_{K}(x)
\end{bmatrix} = P ^{-1} E_{1}^{-1} \begin{bmatrix}
    a_{1}(x) \\
    a_{2}(x) \\
    \vdots \\
    a_{K}(x)
\end{bmatrix} = P ^{-1}  \begin{bmatrix}
    \alpha_{1}(x) \\
    \alpha_{2}(x) \\
    \vdots \\
    \alpha_{K}(x)
\end{bmatrix}= \begin{bmatrix}
    \alpha_{2}(x) \\
    \alpha_{3}(x) \\
    \alpha_{1}(x) \\
    \alpha_{4}(x) \\
    \vdots \\
    \alpha_{K}(x)
\end{bmatrix}.$$
Clearly, $\mathbf{\alpha}(x)$ and $\tilde{\mathbf{\alpha}}(x)$ might not yield the same decision boundary, as $\argmax_k \mathbf{\alpha}_k(x)$ and $\argmax_k \tilde {\mathbf{\alpha}}_k(x)$ may not always be the same. For example if $\alpha_{2}(x) = \max_{k \in [K]} \alpha_{k}(x)$ then with $E_{1}$ the optimal decision is 2  but with $E_{2}$ the optimal decision is 1.

\subsubsection{The imbalanced case}

For $p \neq \frac{\mathbf{1}_{K}}{K}$, we start with an error matrix $E_{1}$ as in Lemma \ref{lemma:stochastic-matrix-existance} such that $E_{1}p = p'$ and $E_1$ is invertible. We let $E_{2} = (1-\delta) E_{1} + \delta p' 1_{K}^{T}$, where $\delta > 0$ is chosen small enough such that $E_2$ remains invertible with positive determinant. Note that $$E_{2}p = (1 - \delta) E_{1} p + \delta p' 1_{K}^{T} p = (1 -\delta) p' + \delta p' = p'.$$ Now, say $\alpha(x) = E_{1}^{-1}a(x)$, then 
\begin{align*}
    \tilde{\alpha}(x) & = \left[(1-\delta) E_{1} + \delta p' 1_{K}^{T}\right]^{-1}a(x)  = \left[ \frac{E_{1}^{-1}}{1-\delta} - \frac{\frac{\delta}{(1-\delta)^{2}}E_{1}^{-1} p' 1_{K}^{T} E_{1}^{-1}}{1 + \frac{\delta}{1 - \delta}1_{K}^{T}E_{1}^{-1}p'}\right] a(x) \\
    & = \left[I_{K} - \frac{\frac{\delta}{1- \delta}E_{1}^{-1}p'1_{K}^{T}}{1 + \frac{\delta}{1 -\delta}1_{K}^{T}E_{1}^{-1}p'} \right] \frac{E_{1}^{-1}a(x)}{1 -\delta} \\
    & = \left[I_{K} - \frac{\frac{\delta}{1-\delta}p1_{K}^{T}}{1 + \frac{\delta}{1-\delta}1_{K}^{T}p} \right] \frac{\alpha(x)}{1-\delta} \\
    & = \frac{1}{1-\delta}\left[ \alpha(x) - \frac{\frac{\delta}{1 -\delta}p 1_{K}^{T}\alpha(x)}{1 + \frac{\delta}{1-\delta}}\right]  = \frac{1}{1-\delta} \left[ \alpha(x) - \delta p \right].
\end{align*} where the second equality is obtained using the Sherman–Morrison identity on the matrix $[(1-\delta) E_{1} + \delta p' 1_{K}^{T}]^{-1}$. 
Since $\tilde{\alpha}(x) = \frac{1}{1-\delta}[\alpha(x) - \delta p]$ and $p \neq \frac{\mathbf{1}_{K}}{K}$ the decision boundaries may not be the same as $\argmax_k \mathbf{\alpha}_k(x)$ and $\argmax_k \tilde {\mathbf{\alpha}}_k(x)$ may be different. 
% Hence, $E_{1}$ and $E_{2}$ might give 2 different optimal decision boundaries.

% \input{sections/weighted-erm}
\section{A comparison of peer loss function and weighted ERM} \label{sec:peer-loss}

\paragraph{A comparison:}
Let us consider a binary classification setup where we  observe the noisy dataset $\{(x_i,  y_i')\}_{i = 1}^n$.
In \citet[Equation (5)]{liu2020peer} the peer loss function is defined as 
\begin{equation}
\label{eq:peer-loss}
    \ell_{\text{peer}}(f(x_i), y'_i) = \ell(f(x_i), y_i') - \ell(f(x_j), y'_k)\,,
\end{equation} where $f$ is a classifier and $j\neq k$ is uniformly drawn from $[n] = \{1, \dots, n\}$. Then the peer risk for the dataset can be written as 
\begin{equation}
   \textstyle \hat L_{\text{peer}} = \frac{1}{n} \sum_{i = 1}^n \left[ \ell(f(x_i), y_i') - \frac{1}{n(n-1)}\sum_{j = 1}^n \sum_{k \neq j} \ell(f(x_j), y'_k) \right]\,.
\end{equation}
As we see in Lemma \ref{lemma:peer-loss-simplification}, the peer loss can be simplified as 
\begin{equation}
   \textstyle  \hat L_{\text{peer}} = \frac{1}{n-1} \sum_{i = 1}^n p_n(1 - y'_i) \big\{ \ell(f(x_i), y_i') - \ell(f(x_i), 1 - y'_i) \big\} \,,
\end{equation} where, for $y \in \{0, 1\}$ the $ p_n(y) = \nicefrac{1}{n} \sum_{i = 1}^n \mathbbm{1}\{y'_i = y\}$ is the proportion samples with noisy label $y$. This is strikingly similar to the class balanced weighted ERM with the weights $w_i = p_n(1 - y'_i)$ which is defined as 
\begin{equation}
  \textstyle    \hat L_{\text{weighted}} = \frac{1}{n}  \sum_{i = 1}^n p_n(1 - y'_i) \ell(f(x_i), y_i') \,.
\end{equation} In fact, for $0$-$1$ loss ($\ell(f(x), y) = \mathbbm{1}\{ f(x) \neq y\}$) the peer risk function and the weighted empirical risks are same up-to a constant adjustment in the loss. 

\paragraph{A failure case of peer-loss:}
If the loss $\ell$ is not bounded, the peer loss may not be bounded below. For example, peer loss for entropy loss ($\ell(a, y) \triangleq -ay + \log\{1 + e^a\}$) simplifies to 
\begin{equation} \label{eq:peer-loss-entropy}
     \textstyle  \hat L_{\text{peer}} = - \frac{1}{n-1} \sum_{i = 1}^n p_n(1 - y'_i) (2y_i' - 1) f(x_i)
\end{equation}  where $f(x)$ is the logit of prediction.  To understand that the peer loss in \eqref{eq:peer-loss-entropy} is may bounded below we consider logistic regression model, \ie, $f(x) = x^\top \beta$, which is a simple model, and often used in many classification tasks. Then  the corresponding peer loss
\begin{equation} \label{eq:peer-loss-logistic}
\textstyle  \hat L_{\text{peer}} = - \frac{1}{n-1} \sum_{i = 1}^n p_n(1 - y'_i) (2y_i' - 1) x_i^\top \beta = - \big\{\frac{1}{n-1} \sum_{i = 1}^n p_n(1 - y'_i) (2y_i' - 1) x_i \big\}^\top \beta
\end{equation}
can diverge to $- \infty$ as long as $\|\beta \|_2 \to \infty$ for any $\beta$ that satisfies \[
\textstyle
\big\{\frac{1}{n-1} \sum_{i = 1}^n p_n(1 - y'_i) (2y_i' - 1) x_i \big\}^\top \frac{\beta}{ \|\beta\|_2}>0\,.\]  
% Moreover, in this case if one restricts $\|\beta\|_2 \le 1$ then the peer-loss optimization simply reduces to the weighted 

\begin{lemma}
\label{lemma:peer-loss-simplification}
For each $i \in [n]$ let us assume that $j \neq k$ is uniformly drawn from $[n]$. Then the peer loss defined in \eqref{eq:peer-loss} simplifies to 
\begin{equation}
   \textstyle \frac{1}{n-1} \sum_{i = 1}^n p_n(1 - y'_i) \big\{ \ell(f(x_i), y_i') - \ell(f(x_i), 1 - y'_i) \big\} \,,
\end{equation} where, for $y \in \{0, 1\}$ the $ p_n(y) = \nicefrac{1}{n} \sum_{i = 1}^n \mathbbm{1}\{y'_i = y\}$ is the proportion samples with $y_i' = y$.

\end{lemma}

\begin{proof}[Proof of Lemma \ref{lemma:peer-loss-simplification}] To simplify the peer loss we begin with he following  equality. 
\begin{equation} \label{eq:b1.1}
    \begin{aligned}
    & \textstyle\frac{1}{n} \sum_{i = 1}^n \Big[ \ell(f(x_i), y_i') - \frac{1}{n(n-1)}\sum_{j = 1}^n \sum_{k \neq j} \ell(f(x_j), y'_k) \Big]\\
    &\textstyle = \frac{1}{n} \sum_{i = 1}^n   \ell(f(x_i), y_i') - \frac{1}{n(n-1)}\sum_{j = 1}^n \sum_{k \neq j} \ell(f(x_j), y'_k)  \\
    & \textstyle= \frac{1}{n} \sum_{i = 1}^n  
  \ell(f(x_i), y_i') - \frac{1}{n(n-1)}\sum_{i = 1}^n \sum_{k \neq i} \ell(f(x_i), y'_k)  \\
    &\textstyle = \frac{1}{n} \sum_{i = 1}^n  \Big[ \ell(f(x_i), y_i') - \frac{1}{(n-1)} \sum_{k \neq i} \ell(f(x_i), y'_k) \Big]
    \end{aligned}
\end{equation} where the third equality is obtained simply replacing the index $j$ with the index $i$. Here, we notice that 
\begin{equation}\label{eq:b1.2}
    \begin{aligned}
    &\textstyle \ell(f(x_i), y_i') - \frac{1}{(n-1)} \sum_{k \neq i} \ell(f(x_i), y'_k)\\
    &\textstyle =  \Big\{1 + \frac1 {n-1}\Big\} \ell(f(x_i), y_i') - \frac{1}{(n-1)} \sum_{k = 1}^n \ell(f(x_i), y'_k)\\
    &\textstyle = \frac{n}{n-1} \ell(f(x_i), y_i') - \frac{n}{n-1} \times  \frac{1}{n} \sum_{k = 1}^n \ell(f(x_i), y'_k) \\
    &\textstyle = \frac{n}{n-1} \big\{ \textstyle\ell(f(x_i), y_i') - \frac{1}{n} \sum_{k = 1}^n \ell(f(x_i), y'_k)\big\}
    \end{aligned}
\end{equation} and that 
\begin{equation} \label{eq:b1.3}
    \begin{aligned}
    &\textstyle \ell(f(x_i), y_i') - \frac{1}{n} \sum_{k = 1}^n \ell(f(x_i), y'_k)\\
    & = \textstyle\ell(f(x_i), y_i') - \frac{1}{n} \sum_{k = 1}^n \ell(f(x_i), y'_i) \mathbbm{1}\{ y_k' = y_i'\\ 
    &\textstyle ~ - \frac{1}{n} \sum_{k = 1}^n \ell(f(x_i), 1 - y'_i) \mathbbm{1}\{ y_k' = 1- y_i'\}\\
    & = \ell(f(x_i), y_i') -  \ell(f(x_i), y'_i) p_n(y_i') -  \ell(f(x_i), 1 - y'_i) p_n(1 - y_i') \\
    & = p_n( 1- y_i') \big\{\ell(f(x_i), y_i') - \ell(f(x_i), 1- y_i')\big\}\,.
    \end{aligned}
\end{equation} Combining \eqref{eq:b1.1}, \eqref{eq:b1.2} and \eqref{eq:b1.3} we have the result. 
\end{proof}
\section{Technical results}
\label{sec:technical-results}
\begin{proof}[Proof of lemma \ref{lemma:reweighting-informal}]
Let us define a distribution $Q$ as $q(x, y, y') = p(x, y, y') w(x, y,y')$ where  $w(x, y, y') = c P(Y' = 1-y')$ for some constant $c > 0$. 
According to lemma \ref{lemma:instance-independence-reweighting} instance independence assumption is still valid for $Q$, \ie, $q(x, y, y') = q(x\mid y) Q(Y = y, Y' = y')$. Thus
\begin{equation}
    \begin{aligned}
    q(x\mid Y = y) Q(Y = y, Y' = y') = p(x\mid Y = y) P(Y = y, Y' = y') c P(Y' = 1-y')\,.
    \end{aligned}
\end{equation} If we integrate both sides with respect to $x$ over the space $\cX$ then we have $\int_{\cX} q(x\mid Y = y) dx = \int_{\cX} p(x\mid Y = y) dx = 1$ and the above equation reduces to 
\begin{equation}
    Q(Y = y, Y' = y') =  P(Y = y, Y' = y') c P(Y' = 1-y')\,.
\end{equation} Now we take a summation over $y$ in the both sides and obtain 
\begin{equation}
    \begin{aligned}
   &\textstyle \sum_{y} Q(Y = y, Y' = y') = \sum_{y} P(Y = y, Y' = y') c P(Y' = 1-y')\\
   \text{or,} ~ & Q(Y' = y') = P(Y' = y') c P(Y' = 1-y') = c p' (1 - p')
    \end{aligned}
\end{equation} where $p' = P(Y' = 1)$. Since, $\sum_{y'}Q(Y' = y') = 1$, from the above equation we obtain $c = \frac{1}{2p'(1 - p')}$ and $w(x, y, y') = \frac1{2 P(Y' = y')}$.

For $y \in \{0, 1\}$ let us define $\alpha_y (x) = q(x \mid Y' = y') Q(Y' = y')$, $a_y = p(x \mid Y = y) P(Y = y)$, $e_0 = P(Y' = 1\mid Y = 0)$, $e_1 = P(Y' = 0\mid Y = 1)$ and notice that 
\begin{equation} \label{eq:rw.1}
   \begin{aligned}
       \alpha_1 (x) & = q(x \mid Y' = 1) Q(Y' = 1) = q(x, 1, 1) + q(x, 0, 1)\\
       & =\textstyle \frac{p(x, 1, 1)}{2 p'} + \frac{p(x, 0, 1)}{2 p'}\\
       & = \textstyle\frac{p(x\mid Y =  1) P(Y = 1)P(Y' = 1\mid Y = 1)}{2 p'}+ \frac{p(x\mid Y =  0) P(Y = 0)P(Y' = 1\mid Y = 0)}{2 p'}\\
       & = \textstyle \frac{a_1(x) \nicefrac{1}{2} (1 - e_1)}{ 2 p'} + \frac{a_0(x) \nicefrac{1}{2} e_0}{ 2 (1- p')} = \frac{a_1(x)  (1 - e_1)}{ 4 p'} + \frac{a_0(x)  e_0}{ 4 p'}\,.
   \end{aligned} 
\end{equation}
Similarly, we obtain 
\begin{equation} \label{eq:rw.2}
    \textstyle \alpha_0(x) = \frac{a_1(x)  e_1}{ 4 (1- p')} + \frac{a_0(x)  (1 - e_0)}{ 4 (1 - p')}\,.
\end{equation} Taking the differences between \eqref{eq:rw.1} and \eqref{eq:rw.2} we obtain 
\begin{equation}
 \textstyle   \alpha_1(x) - \alpha_0(x) = \frac{a_1(x)}{2} \left ( \frac{  1 - e_1}{ 2 p'} - \frac{e_1}{ 2(1 -  p')}\right)  + \frac{a_0(x)}{2} \left( \frac{  e_0}{ 2 p'} - \frac{1 - e_0}{ 2(1 -  p')} \right)\,.
\end{equation} Here, we use 
\begin{equation}
    p' = P(Y' = 1) = \nicefrac{1}{2} (1 - e_1) + \nicefrac{1}{2} e_0  \implies  2p' = 1 - e_1 + e_0,\ ~ \text{and} ~ 2( 1- p') = 1 - e_0 + e_1
\end{equation} in the above equation and obtain 
\begin{equation}
    \begin{aligned}
        \textstyle \alpha_1(x) - \alpha_0(x) & = \textstyle \frac{a_1(x)}{2} \left ( \frac{  1 - e_1}{ 2 p'} - \frac{e_1}{ 2(1 -  p')}\right)  + \frac{a_0(x)}{2} \left( \frac{  e_0}{ 2 p'} - \frac{1 - e_0}{ 2(1 -  p')} \right)\\
        & =\textstyle \frac{a_1(x)}{2} \left ( \frac{  1 - e_1}{ 1 - e_1 + e_0} - \frac{e_1}{ 1 - e_0 + e_1}\right)  + \frac{a_0(x)}{2} \left( \frac{  e_0}{ 1 - e_1 + e_0} - \frac{1 - e_0}{ 1 - e_0 + e_1} \right)\\
        & =\textstyle \frac{a_1(x)}{2} \left ( \frac{  1 - e_1}{ 1 - e_1 + e_0} - \frac{e_1}{ 1 - e_0 + e_1}\right)  + \frac{a_0(x)}{2} \left( 1 - \frac{  1 - e_1}{ 1 - e_1 + e_0} - 1 +  \frac{e_1}{ 1 - e_0 + e_1} \right)\\
        & =\textstyle \frac{a_1(x)}{2} \left ( \frac{  1 - e_1}{ 1 - e_1 + e_0} - \frac{e_1}{ 1 - e_0 + e_1}\right)  + \frac{a_0(x)}{2} \left(  - \frac{  1 - e_1}{ 1 - e_1 + e_0}  +  \frac{e_1}{ 1 - e_0 + e_1} \right)\\
        & =\textstyle \frac{a_1(x) - a_0(x)}{2} \left ( \frac{  1 - e_1}{ 1 - e_1 + e_0} - \frac{e_1}{ 1 - e_0 + e_1}\right) = \textstyle \frac{a_1(x) - a_0(x)}{2}  \left ( \frac{1 - e_1 - e_0}{(1 - e_1 + e_0) (1 - e_0 + e_1)}\right)\,.
    \end{aligned}
\end{equation} Since, $1 - e_1 + e_0, 1 - e_0 + e_1 \ge 1 - e_1 - e_0 > 0$ we have $\frac{1 - e_1 - e_0}{(1 - e_1 + e_0) (1 - e_0 + e_1)} > 0$. Now we see that $a_1(x) \ge a_0(x)$ if and only if $\alpha_1(x) \ge \alpha_0(x)$. Now, noticing that (1) $P(Y = 1\mid X = x) \ge \nicefrac{1}{2}$ if and only if $a_1(x) \ge a_0(x)$, and (2) $Q(Y' = 1\mid X = x) \ge \nicefrac{1}{2}$ if and only if $\alpha_1(x) \ge \alpha_0(x)$ we have 
\[
\big\{x: Q(Y' = 1\mid X = x) \ge  \nicefrac{1}{2} \big\}  = \big\{x: P(Y = 1\mid X = x) \ge  \nicefrac{1}{2} \big\}\,.
\] This implies the Bayes decision boundaries for $Y'$ on the $q(x, y, y')$ distribution and for $Y$ on $P$ distribution are same and \begin{equation}
    f^\star_P(x)= \argmin_{f \in \mathcal{F}} \Ex_P[p'(1 - Y') \ell(f(X), Y')]\,.
\end{equation}
\end{proof}

\begin{lemma}[Reweighting of the noisy labels] \label{lemma:instance-independence-reweighting}
    Say $P \in \cQ(K,p)$ satisfying the conditional independence property that $P_{Y'|Y,X} = P_{Y'|Y}$ and $p' = P_{y}$. Then the sample can be reweighted to obtain class balanced $Y'$ that is $p' = P_{y} = \frac{1_{k}}{k}$ while satisfying the conditional independence $X \perp Y' \mid Y $ condition for the reweighted distribution.
\end{lemma}

\begin{proof}
After the reweighting 
\[ \textstyle
q(x,y,y')  = p(x,y,y') \frac{1}{k p(y')}  = p(x|y) \frac{p(y,y')}{kp(y')}  = \frac{p(x|y)p(y|y') }{k} 
\]
the $Y'$ gets class balanced, as seen below.
\begin{align*}
    q(y') & = \textstyle\sum_{y}\int_{x,y}q(x,y,y') dx   = \textstyle\frac{1}{k}\sum_{y}\int_{x,y} p(x|y) p(y|y') dx  \\
    & = \textstyle\frac{1}{k} \sum_{y} \left(\int_{x} p(x|y) dx \right)p(y|y')   = \textstyle\frac{1}{k} \sum_{y} p(y|y')   = \frac{1}{k}
\end{align*}
For such reweighting the conditional independence of $X \perp Y' \mid Y $ still is satisfied for $q$:
\[
\textstyle q(x|y,y') = \frac{q(x,y,y')}{q(y,y')} = \frac{\frac{1}{k} p(x|y) p(y|y')}{\frac{1}{k} p(y|y')} = p(x|y).
\]
\end{proof}

\begin{lemma} \label{lemma:stochastic-matrix-existance}
    Let $p$ and $p'$ be any probability vectors on the space $[K]$. Let us assume that the entries of $p$ and $p'$ are all positive then there exists a matrix $E$ such that (1) its entries are non-negative, (2) the column sums are all one, (3) the determinant is positive, and (4) $p' = Ep$. 
\end{lemma}

\begin{proof}[Proof of lemma \ref{lemma:stochastic-matrix-existance}]
Let us assume that $P \in \reals^{K \times K}$ is the permutation matrix that reorders the $p'-p$ in a decreasing fashion, \ie, the entries of $P(p'-p)$ are decreasing. Note that $Pp$ and $Pp'$ are still probability vectors that have positive entries. Let us define $K_1 = \max\{k : [Pp]_k \le [Pp']_k\}$. Now we define our stochastic matrix $\tilde E = [\tilde e_{ij}]_{i, j\in [K]}$ as the following. 
\begin{equation}
 \tilde e_{i, j}  =    \begin{cases}
    1 & \text{if} ~ i = j \le K_1, \\
    \nicefrac{[Pp']_i}{[Pp]_i} & \text{if} ~ i = j \ge K_1 +1, \\
    \frac{[P(p' -p)]_i [P(p - p')_j]}{ [Pp]_j\sum_{k = 1}^{K_1}[P(p' -p)]_k} & \text{if} ~ i \le K_1, ~  j \ge K_1 +1,\\
    0 & \text{elsewhere}
    \end{cases}
\end{equation} 

Note that, for $i \le K_1$
\[
\begin{aligned}
\textstyle\sum_{j = 1}^K \tilde e_{ij} [Pp]_j & =\textstyle \tilde e_{ii} [Pp]_i + \sum_{j \ge K_1 + 1} \tilde e_{ij} [Pp]_j \\
& = \textstyle1\times {[Pp]_i} +  \sum_{j \ge K_1 +1 } \frac{[P(p' -p)]_i [P(p - p')_j]}{ \cancel{[Pp]_j}\sum_{k = 1}^{K_1}[P(p' -p)]_k} \times  \cancel{[Pp]_j}\\
& =\textstyle [Pp]_i + [P(p' -p)]_i \times \frac{\sum_{j = K_1 +1 }^{K}  [P(p - p')_j]}{\sum_{k = 1}^{K_1}[P(p' -p)]_k}
\end{aligned}
\] If it holds (to be established later)
\begin{equation}\label{eq:eq-tech1}
   \textstyle \sum_{j = K_1 +1 }^{K}   [P(p - p')_j] = \sum_{k = 1}^{K_1}[P(p' -p)]_k
\end{equation} then for $i \le K_1$ we have
\[\textstyle
\sum_{j = 1}^K \tilde e_{ij} [Pp]_j = [Pp']_i\,.
\] Additionally for $i \ge K_1 +1 $ we have 
\[ \textstyle
\sum_{j = 1}^K \tilde e_{ij} [Pp]_j = e_{ii} [Pp]_i = \frac{[Pp']_i}{[Pp]_i} \times [Pp]_i = [Pp']_i\,.
\] This implies $Pp' = \tilde E Pp$ or $p' = P^{-1} \tilde E Pp$. Since the permutation matrix $P$ is also orthogonal, we have $P^{-1} \tilde E P = P^\top \tilde E P$. We define $E = P^\top \tilde E P$ Then $p' = Ep$, which verifies (4).  Note that $E$ is obtained simply by permuting the rows and columns of $\tilde E$ according to the permutation matrix $P^\top$.

Clearly, (1) is satisfied because the entries of $\tilde E$ and hence of $E$ are non-negative. 

We verify (2) for $\tilde E$, which implies the same for $E$, because 
\[
\mathbf{1}^\top  E  = \mathbf{1}^\top P^\top \tilde E P = \mathbf{1}^\top \tilde E P = \mathbf{1}^\top  P = \mathbf{1}^\top\,.
\]
We notice that for $j \le K_1$ the $j$-th column is simply $e_j$ (the $j$-th canonical basis of $\reals^K$) which has column sum one. If $j \ge K_1 +1$ then the column sum is 
\[
\begin{aligned}
\textstyle\sum_{i = 1}^K \tilde e_{i, j} & = \textstyle \sum_{i \le K_1}  \frac{[P(p' -p)]_i [P(p - p')_j]}{ [Pp]_j\sum_{k = 1}^{K_1}[P(p' -p)]_k} + \frac{[Pp']_j}{[Pp]_j} =\textstyle   \frac{\cancel{\Big(\sum_{i \le K_1}[P(p' -p)]_i\Big)} [P(p - p')_j]}{ [Pp]_j\cancel{\sum_{k = 1}^{K_1}[P(p' -p)]_k}} + \frac{[Pp']_j}{[Pp]_j}\\
& =\textstyle \frac{[P(p - p')_j] + [Pp']_j}{[Pp]_j} = 1\,.
\end{aligned}
\]

To verify (3) we first notice that $\tilde E$ and $E = P^\top \tilde E P$ have the same determinant, since $P$ is an orthogonal matrix. So, we shall only prove that $\tilde E$ have positive determinant. Now, we notice that $\tilde E$ is a upper triangular matrix, whose determinant is equal the product of the diagonal entries, \ie,  $\prod_{i \ge K_1 + 1} \nicefrac{[Pp']_i}{[Pp]_i}$. Since $Pp$ and $Pp'$ have positive entries, the determinant is positive.

It remains to verify \eqref{eq:eq-tech1}. Since 
$
\sum_{i = 1}^K [P(p'-p)]_i = 0
$ we have 
\[ \textstyle
\sum_{i = 1}^{K_1} [P(p'-p)]_i = - \sum_{i = K_1}^{K} [P(p'-p)]_i = \sum_{i = K_1}^{K} [P(p-p')]_i
\] which verifies  \eqref{eq:eq-tech1}.
\end{proof}

\end{document}